\documentclass[runningheads]{llncs}

\usepackage{eccv}

\usepackage{eccvabbrv}

\usepackage{graphicx}
\usepackage{booktabs}

\usepackage[accsupp]{axessibility}  

\usepackage{hyperref}

\usepackage{orcidlink}

\usepackage{algorithm}
\usepackage{algorithmic}

\def \R {\mathbb{R}}
\def \x {\mathbf{x}}
\def \z {\mathbf{z}}
\def \w {\mathbf{w}}
\def \LL {\mathcal{L}}
\def \OO {\mathcal{O}}

\newtheorem{thm}{Theorem}
\newtheorem{prop}{Proposition}

\definecolor{da}{gray}{0.6}

\begin{document}

\title{Online Zero-Shot Classification with CLIP} 


\author{Qi Qian\inst{1}\orcidlink{0009-0007-8661-1169}\thanks{Corresponding author} \and
Juhua Hu\inst{2}\orcidlink{0000-0001-5869-3549}}

\authorrunning{Q.~Qian and J.~Hu}

\institute{Alibaba Group, Bellevue, WA 98004, USA \and
School of Engineering and Technology, \\University of Washington, Tacoma, WA 98402, USA\\
\email{qi.qian@alibaba-inc.com, juhuah@uw.edu}}

\maketitle

\begin{abstract}
  Vision-language pre-training such as CLIP enables zero-shot transfer that can classify images according to the candidate class names. While CLIP demonstrates an impressive zero-shot performance on diverse downstream tasks, the distribution from the target data has not been leveraged sufficiently. In this work, we study a novel online zero-shot transfer scenario, where each image arrives in a random order for classification and is visited only once to obtain prediction immediately without storing its representation. Compared with the vanilla zero-shot classification, the proposed framework preserves its flexibility for online service while considering the statistics of the arrived images as the side information to capture the distribution of target data, which can help improve the performance of real-world applications. To tackle the challenge of effective online optimization, we first develop online label learning to model the target data distribution. Then, the proxy of each class in the vision space is further optimized with the proposed online proxy learning method to mitigate the modality gap between images and text. The convergence of both online strategies can be theoretically guaranteed. By combining the predicted label from the online label learning and proxy learning, our online zero-shot transfer method (OnZeta) achieves $78.94\%$ accuracy on ImageNet without accessing the entire data set. Moreover, extensive experiments on other 13 downstream tasks with different vision encoders show a more than $3\%$ improvement on average, which demonstrates the effectiveness of our proposal. Code is available at \url{https://github.com/idstcv/OnZeta}.
\end{abstract}

\section{Introduction}
\label{sec:intro}

Vision-language pre-training has attracted much attention recently due to its impressive zero-shot transfer performance on various downstream tasks. The desired property mainly comes from aligning the vision and text space. For example, one of the most prevalent pre-training methods, \ie, CLIP~\cite{clip}, consists of distinct vision and text encoders for learning image and text representations, respectively. These encoders are optimized by minimizing a contrastive loss defined on image-text pairs. The loss aims to pull images and their corresponding text descriptions together while pushing the irrelevant text or images away~\cite{softtriple}. 

\begin{figure}[t]
\centering
\includegraphics[height = 1.25in]{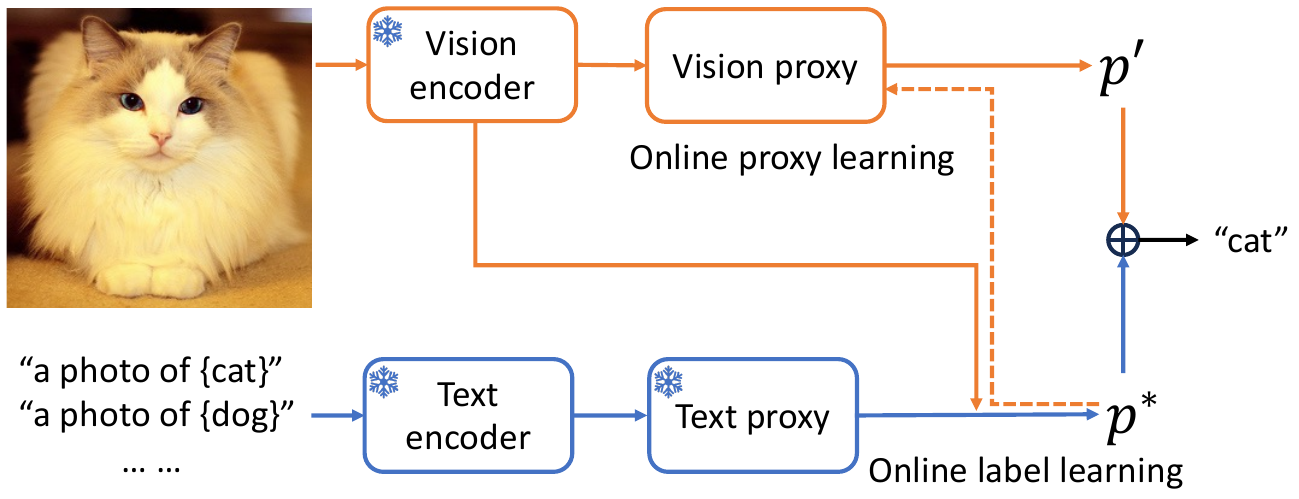}
\caption{Illustration of the proposed online zero-shot transfer method (OnZeta). Blue and orange lines denote the inference in the text and vision space, respectively. By incorporating the predictions from online label learning and online proxy learning, OnZeta can leverage the biased prediction from the text space to reduce the variance in the target vision space in an online manner.}\label{fig:illu}
\end{figure}

After the alignment between vision and text, a novel zero-shot classification paradigm emerges to transfer knowledge in pre-trained models to diverse tasks. Concretely, given all class names from the target task, each class can be represented by a class proxy that applies the text encoder in CLIP for the corresponding class name, which is also termed a text proxy. When an image arrives, its representation is extracted by the vision encoder of CLIP. Then, the label can be assigned by the 1-nearest neighbor (1-NN) classifier with the set of text proxies as the reference. Without any fine-tuning, this simple strategy achieves $77.02\%$ accuracy on ImageNet~\cite{imagenet} using the pre-trained ViT~\cite{vit} as the vision encoder.

Since the success of CLIP, many methods have been developed to explore different side information from the target task to further improve the transfer performance. Firstly, if a few labeled images are available from the target task, the zero-shot transfer can be formulated as a few-shot learning problem and the input text prompt can be learned with the labeled vision data by fine-tuning~\cite{coop}. Secondly, if the computational resource is too limited to do fine-tuning, the labeled data can be cached as a vision classifier to classify images directly~\cite{tip}.
However, supervised information is often rare, while unsupervised information is more accessible. Therefore, many efforts have been devoted to exploiting the unlabeled images to enhance the zero-shot performance. \cite{tpt} generates multiple augmented examples from a given image, and then learns the specific text prompt for each image as few-shot learning accordingly. \cite{inmap} considers reconstructing the class proxy in the vision space, \ie, vision proxy, by leveraging a set of unlabeled images in proxy learning. While \cite{inmap} shows superior zero-shot performance, it optimizes proxies in an offline manner. Due to the privacy issue, gathering or keeping a set of unlabeled images is still challenging in some real-world applications.

Therefore, in this work, we investigate a novel practical scenario for zero-shot transfer, termed online zero-shot transfer. Concretely, when streaming images arrive as in conventional zero-shot learning, the model has to classify the image immediately without refining. After that, the class proxy can be updated but the representations of arrived images cannot be kept. Unlike \cite{inmap} that can access the whole unlabeled set, this online setting is way more challenging, where only statistics of seen images can be leveraged for optimization and each image is visited only once. Thereafter, the standard iterative optimization in \cite{inmap} is not applicable, while the problem (\ie, online service to classify each unlabeled image on the fly without storing data) is ubiquitous in real-world applications (\eg, mobiles and robots).

\begin{figure}[t]
\centering
\includegraphics[height = 1.1in]{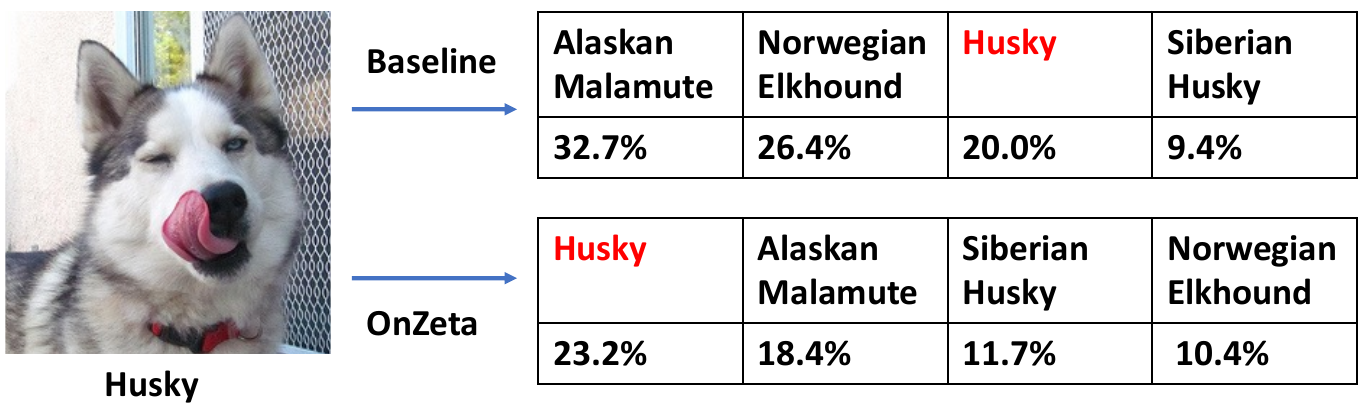}
\caption{Top 4 predicted labels with corresponding probabilities from baseline zero-shot method by CLIP and our proposal.}\label{fig:illu_example}
\end{figure}

To address the online service challenge, we first consider exploring the data distribution of the target task to improve the prediction from the text proxy. Note that the vanilla zero-shot prediction focuses on individual images and ignores the distribution over the entire data set. To capture the distribution of all data, we develop an online label learning algorithm to balance the assignment between different classes globally. In addition, the class proxy from the text (\ie, text proxy) can be biased for the vision data due to the modality gap, and thus we propose online proxy learning to obtain the class proxy in the vision space (\ie, vision proxy) directly in an online manner. Finally, the prediction from both the text space and vision space can be combined to help reduce the variance of learning in the vision space as illustrated in the framework of our online zero-shot transfer (OnZeta) shown in \cref{fig:illu}. More importantly, with online optimization, our proposal OnZeta can leverage the information from target data effectively and efficiently, which helps correct the prediction compared to the baseline zero-shot method. For example, \cref{fig:illu_example} shows that an image of ``Husky'' will be misclassified by vanilla CLIP. By considering the distribution over the whole data set and incorporating the learned vision proxy, our method can obtain the right label in an online manner. The main contributions of this work can be summarized as follows.

\begin{itemize}
\item Considering the practical need for online service (\ie, classifying each image on the fly without storage) and improving the performance of vanilla zero-shot classification, we study a novel setting, that is, online zero-shot transfer, using pre-trained vision-language models.
\item To address the online service challenge, we propose an online label learning method to trace the distribution of assigned labels over the entire data in an online manner. Moreover, an online vision proxy learning method is proposed to further improve the class proxy in vision space for classification. The convergence of both methods can be theoretically guaranteed.
\item Experiments on 14 downstream tasks demonstrate the effectiveness of our proposal OnZeta. Concretely, when visiting each image only once, OnZeta outperforms the baseline by a margin of more than $3\%$ on average without storing any unlabeled data.
\end{itemize}

\section{Related Work}
\label{sec:related}
CLIP consists of two encoders (\ie, vision and text), and the optimization can be conducted on each of them individually. In this section, we briefly review research efforts in these two directions. 

\subsection{Optimization with Text Encoder}
CLIP demonstrates that the zero-shot performance heavily depends on the text prompts for generating appropriate class proxies, and many methods investigate learning the optimal text prompt as input to the text encoder. CoOp~\cite{coop} marks the input text prompts as learnable variables and optimizes them via few-shot learning with a limited number of labeled images for each class. CoCoOp~\cite{cocoop} further includes the vision representation as the conditional context for prompt learning. TPT~\cite{tpt} relaxes the requirement of labeled examples and optimizes the consistency between different augmentations from the same image for prompt learning. However, it has to optimize the prompt for each image with $64$ augmentations, which is inefficient for computation and memory, making it infeasible for real-time applications as the online scenario in this work.

\subsection{Optimization with Vision Encoder}
While prompt learning can improve the performance on the target data, recent research shows that optimizing the vision representation can be more effective for classification~\cite{clipadapter,inmap}. CLIP-Adapter~\cite{clipadapter} introduces an additional adapter network after the vision encoder and optimizes the vision representation with few-shot learning. Although a similar adapter can be attached for the text encoder, the empirical study shows that optimizing the vision space is sufficient for transfer and can achieve better performance than either the version with only a text adapter or that with both vision and text adapters. \cite{robustft} shows that fine-tuning the vision classifier directly with few-shot learning can achieve better performance than CoOp. Recently, the analysis in \cite{inmap} demonstrates that the optimal class proxy for visual classification lives in the vision space. To mitigate the modality gap, \cite{inmap} proposes proxy learning that aims to learn class proxy for each class in vision space. However, it requires a set of unlabeled data and will learn proxies iteratively. In this work, we also aim to refine the class proxy in the vision space, but with the learned labels in an online manner, which is more applicable for online real applications.

\section{Online Zero-shot Transfer}
\label{sec:method}

\subsection{Zero-shot Transfer in CLIP}
Given a data set consisting of image-text pairs as $\{I_i, T_i\}_{i=1}^n$, where $I_i$ denotes an image and $T_i$ is the corresponding text such as class names or captions, CLIP trains vision and text encoders jointly to align the vision subspace and its corresponding text subspace. Let $f(\cdot)$ and $g(\cdot)$ denote the vision encoder and text encoder, respectively. The representations for images and text can be extracted as $\x_i = f(I_i)$ and $\z_i = g(T_i)$.

After pre-training two encoders, zero-shot classification can be implemented by the 1-nearest neighbor (1-NN) classifier. Concretely, given the class names of the target task, the proxy of $j$-th class $\z_j$ can be obtained by applying the text encoder for the prompt as ``a photo of a \{class name\}''. Then, the image $I_i$ will be classified as
\begin{eqnarray}\label{eq:vanilla}
y_i = \arg\max_j \x_i^\top \z_j
\end{eqnarray}

While the conventional zero-shot paradigm demonstrates an impressive transfer performance, the side information from the target task has not been explored sufficiently. Recent work shows that with only a set of unlabeled target data, the zero-shot performance can be improved dramatically~\cite{inmap}. In this work, we further relax the requirement of unlabeled target data and investigate a novel online zero-shot transfer scenario, where each unlabeled image arrives in an online manner without storing.

\subsection{Online Label Learning}
First, we develop an online label learning method to leverage the information from arrived images to capture the distribution over the entire set. 

Given text proxy ${\z_j}$, the predicted distribution over classes for $I_i$ can be computed as
\begin{eqnarray}
q_{i,j} = \frac{\exp(\x_i^\top \z_j/\tau_T)}{\sum_k^C \exp(\x_i^\top \z_k/\tau_T)}
\end{eqnarray}
where $\tau_T$ is the temperature optimized by CLIP.

If the whole image set is available, the label can be refined by considering the distribution over classes. Let $\{p_i\}$ denote the learnable labels, and the optimization problem with $C$ classes can be cast as
\begin{eqnarray}\label{eq:alpha}
\min_{p_i\in\Delta} \frac{1}{n}\sum_iD_{KL}(p_i||q_i)\quad\quad s.t.\quad \forall j,\quad \frac{1}{n}\sum_i p_{i,j}\geq \frac{\alpha}{C} 
\end{eqnarray}
where $p_i = [p_{i,1},\dots,p_{i,C}]$, $q_i = [q_{i,1},\dots,q_{i,C}]$ and $\Delta$ denotes the simplex as $\Delta = \{p_i\in\R^{C}|\sum_j^C p_{i,j}=1;\forall j, p_{i,j}\geq 0\}$. The constraint bounds the minimal size of each class to avoid collapsing and $\alpha\in[0,1]$ is the ratio for balancing the assignment between different classes. $\alpha=1$ denotes a well-balanced distribution that each class has the same proportion of images, while $\alpha=0$ keeps the prediction of the original zero-shot classification from $q_i$.

The problem can be solved effectively given all examples, but optimizing the objective with streaming images is challenging. To address the online problem, we introduce the Lagrange dual variables~\cite{convex} $\{\rho_j\}$ for constraints and rewrite the problem as
\begin{align}\label{eq:dual}
\LL(\rho, p_i) = \max_{\rho_j\geq 0}\min_{p_i\in\Delta} \frac{1}{n}\sum_iD_{KL}(p_i||q_i) - \sum_j^C \rho_j (\frac{1}{n}\sum_i p_{i,j} - \frac{\alpha}{C} )
\end{align}

When the $i$-th example arrives for prediction, the sub-problem with fixed dual variables becomes 
\begin{eqnarray}\label{eq:label}
\LL'(p_i) = \min_{p_i\in\Delta} D_{KL}(p_i||q_i) - \sum_j^C \rho_j^{i-1} p_{i,j}
\end{eqnarray}
where $\rho_j^{i-1}$ indicates the dual variable used for the $(i-1)$-th arrived image.

Fortunately, this problem has the closed-form solution as follows.
\begin{prop}\label{prop:1}
The optimal solution to the problem in Eqn.~\ref{eq:label} is
\begin{eqnarray}\label{eq:psolution}
p_{i,j}^* = \frac{q_{i,j}\exp(\rho_j^{i-1})}{\sum_k^C q_{i,k}\exp(\rho_k^{i-1})}
\end{eqnarray}
\end{prop}
\begin{proof}
It is from the K.K.T. condition~\cite{convex}.
\end{proof}
\paragraph{Remark} Proposition~\ref{prop:1} shows that the dual variable helps balance the assignment between classes. When $\rho=0$ without optimization, it degenerates to the original prediction from CLIP.

\begin{algorithm}[t]
\caption{\textbf{On}line \textbf{Lab}el Learning (OnLab)}
\begin{algorithmic}[1]
\STATE {\bf Input:} Unlabeled image set $\{I_i\}$, class names $\{T_j\}$, pre-trained CLIP encoders $(f,g)$, dual variables $\{\rho_j\}$, ratio $\alpha$, temperature $\tau_T$, initial learning rate $c_\rho$
\STATE Obtain representation of text proxy as $\z_j = g(T_j)$
\STATE Initialize $\forall j, \rho_j = 0$
\WHILE{the $i$-th image arrives} 
\STATE Obtain vision representation as $\x_i = f(I_i)$
\STATE Observe the probability $p_i^*$ by Eqn.~\ref{eq:psolution}
\STATE Have the prediction as $y_i=\arg\max_j p_{i,j}^*$
\STATE Update dual variables as in Eqn.~\ref{eq:dualupdate}
\ENDWHILE
\end{algorithmic}\label{alg:onlab}
\end{algorithm}

With the prediction for the image, the dual variable can be updated by gradient ascent as
\begin{eqnarray}\label{eq:dualupdate}
\rho_j^{i} = \max\{0,\rho_j^{i-1} - \eta_\rho^i(p_{i,j}^* - \alpha/C)\}
\end{eqnarray}
where $\eta_\rho^i$ is the learning rate for the dual variable. When $\{\rho_j\}$ is initialized as $0$ and $\alpha=0$, dual variables will not be updated, which leads to the original prediction.

Our \textbf{on}line \textbf{lab}el learning (OnLab) algorithm is summarized in Alg.~\ref{alg:onlab}. The prediction will be obtained by a close-form solution and the updating for dual variables is also efficient, which makes the proposed method applicable for real-time applications.

Besides the efficiency, the effectiveness of OnLab can be theoretically guaranteed in the following theorem. The detailed proof can be found in the appendix.
\begin{thm}\label{thm:1}
By running Alg.~\ref{alg:onlab} for solving the problem in Eqn.~\ref{eq:dual} with the step size $\eta_\rho^i = c_\rho/\sqrt{i}$, where $c_\rho$ is a constant, the convergence can be guaranteed as
\begin{eqnarray}
\frac{1}{n}\max_\rho\sum_i \LL(\rho, p_i^*) - \min_{p_i}\sum_i\LL(\rho^{i-1}, p_i)\leq \OO(\frac{1}{\sqrt{n}})
\end{eqnarray}
\end{thm}

\subsection{Online Proxy Learning}
\label{sec:onproxy}
After optimizing the label by online label learning, we consider reconstructing the class proxy in the vision space to reduce the modality gap from the text proxy as suggested in \cite{inmap}. Given the whole set of images, vision proxy learning can be cast as an optimization problem
\begin{eqnarray}\label{eq:proxy}
\min_{\w} \LL(\w) = -\sum_i\sum_j p_{i,j}^*\log(p'_{i,j})
\end{eqnarray}
where $p_{i,j}^*$ is the output of online label learning and
\begin{eqnarray}\label{eq:pw}
p'_{i,j} = \frac{\exp(\x_i^\top \w_j/\tau_I)}{\sum_{k} \exp(\x_i^\top \w_k/\tau_I)}
\end{eqnarray}
$\w_j$ denotes the learnable vision proxy for the $j$-th class. $\tau_I$ is the temperature that should be larger than $\tau_T$ to obtain the class proxy in the vision space~\cite{inmap}.

When solving the problem in an online manner, only a single example will be received at each iteration. Therefore, the vision proxy will be updated according to the gradient from the $i$-th example as
\begin{eqnarray}\label{eq:w}
\w^{i} = \Pi_{\|\w\|_2=1}(\w^{i-1} - \eta_w^i \nabla \LL(\w^{i-1}, \x_i))
\end{eqnarray}
where $\Pi(\cdot)$ projects the updated vision proxy to the unit norm. Note that the loss function is convex in $\w$, and the convergence can be guaranteed by the standard online learning theory~\cite{online}.

\begin{thm}\label{thm:2}
By updating $\w$ as in Eqn.~\ref{eq:w} and setting $\eta_w^i = c_w/\sqrt{i}$, where $c_w$ is a constant, the regret can be bounded as
\begin{eqnarray}
\frac{1}{n}\sum_i \LL(\w^{i-1}, \x_i) - \min_{\w} \LL(\w)\leq \OO(\frac{1}{\sqrt{n}}) 
\end{eqnarray}
\end{thm}

Since images arrive in a random order and only a single image will be leveraged for updating $\w$, the variance in learning can be large, especially for the first few iterations. Unlike offline optimization in \cite{inmap}, where the performance is evaluated after the final iteration, online learning has to accumulate the accuracy of predictions from different timestamps. Therefore, we consider combining the pseudo label $p_i^*$ from online label learning in text space to stabilize the prediction.

\begin{algorithm}[t]
\caption{\textbf{On}line \textbf{Ze}ro-shot \textbf{T}r\textbf{a}nsfer (OnZeta)}
\begin{algorithmic}[1]
\STATE {\bf Input:} Unlabeled image set $\{I_i\}$, class names $\{T_j\}$, pre-trained CLIP encoders $(f,g)$, dual variables $\{\rho_j\}$, vision proxy $\{\w_j\}$, ratio $\alpha$, temperature $\tau_T$ and $\tau_I$, initial learning rate $c_\rho$ and $c_w$, weight parameter $\beta$
\STATE Obtain representation of text proxy as $\z_j = g(T_j)$
\STATE Initialize $\forall j, \rho_j = 0$ and $\w_j = \z_j$
\WHILE{the $i$-th image arrives} 
\STATE Obtain vision representation as $\x_i = f(I_i)$
\STATE Observe the probability $p_i^*$ by Eqn.~\ref{eq:psolution}
\STATE Observe the probability $p_i'$ by Eqn.~\ref{eq:pw}
\STATE Have the prediction as $y_i=\arg\max_j \tilde{p}_{i,j}$ where $\tilde{p}_{i}$ is from Eqn.~\ref{eq:lambda}
\STATE Update dual variables as in Eqn.~\ref{eq:dualupdate}
\STATE Update vision proxy as in Eqn.~\ref{eq:w}
\ENDWHILE
\end{algorithmic}\label{alg:onzeta}
\end{algorithm}

For the $i$-th example, assume that $\hat{p}_i$ is the ground-truth distribution in the vision space. Since $p_i^*$ is mainly from the text proxy, it is a biased estimation from the text space. While $p_i'$ is estimated by the vision proxy, the variance can be large due to the online updating. Therefore, these predictions can be mixed to trade-off bias from text space and variance from vision space. Concretely, let \begin{eqnarray}\label{eq:lambda}
\tilde{p}_i = \lambda_i p_i' + (1-\lambda_i)p_i^*
\end{eqnarray}
then we have
\begin{align}
&E[\|\tilde{p}_i - \hat{p}_i\|_2^2] \nonumber \\
&= E[\lambda_i^2\|p_i' - \hat{p}_i\|_2^2 + (1-\lambda_i)^2\|p_i^* - \hat{p}_i\|_2^2+2\lambda(1-\lambda)\langle p_i' - \hat{p}_i,p_i^* - \hat{p}_i\rangle]
\end{align}

For the sake of simplicity, we assume that $p_i'$ is an unbiased estimation for $\hat{p}_i$ such that $E[p_i'] = \hat{p}_i$. Then
\begin{eqnarray}
    E_w[\|\tilde{p}_i - \hat{p}_i\|_2^2] =\lambda_i^2 \underbrace{E[\|p_i' - \hat{p}_i\|_2^2]}_{\mathrm{variance}} + (1-\lambda_i)^2\underbrace{\|p_i^* - \hat{p}_i\|_2^2}_{\mathrm{bias}}
\end{eqnarray}
To minimize the approximation error, the trade-off ratio $\lambda_i$ has the optimal solution as
\begin{eqnarray}
    \lambda_i^* = \frac{\|p_i^* - \hat{p}_i\|_2^2}{\|p_i^* - \hat{p}_i\|_2^2+E[\|p_i' - \hat{p}_i\|_2^2]}
\end{eqnarray}

We find that the bias can be considered as a constant but the variance will be reduced with the learning of $\w$, which implies a monotonically increasing function for combining the pseudo labels from different spaces. While various weight functions for $\lambda$ can be developed according to the analysis, for the $i$-th iteration, we empirically set 
\begin{eqnarray}\label{eq:beta}
\lambda_i = \beta\sqrt{i/n}
\end{eqnarray}
where $\beta\in[0,1]$ is a constant. The combination strategy helps online transfer when the vision proxy is not trained sufficiently.

Incorporating online label learning and online proxy learning, the proposed \textbf{on}line \textbf{ze}ro-shot \textbf{t}r\textbf{a}nsfer (OnZeta) algorithm is summarized in Alg.~\ref{alg:onzeta}. All operations in the proposed method can be computed efficiently without significantly increasing the cost of conventional zero-shot transfer.

\section{Experiments}
\label{sec:experiments}
To evaluate the proposed method, we conduct experiments on 14 diverse downstream tasks. Considering that text prompts are essential for obtaining the appropriate text proxy for zero-shot transfer, we follow the common practice of including an ensemble of $7$ text prompts as suggested in \cite{clip} to generate the text proxy for baseline and our method. The initial learning rates for dual variables and vision proxy are set as $c_\rho=20$ and $c_w=0.5$, respectively. The learning rate will be decayed as $\eta_\rho^i = c_\rho/\sqrt{i}$ and $\eta_w^i = c_w/\sqrt{i}$ where $i$ denotes the order of the incoming image. Unlabeled images will arrive in a random order. For the temperature, we fix $\tau_T=0.01$ that is the optimized value in CLIP and $\tau_I=0.04$ as suggested in \cite{inmap}. All experiments are implemented on a single V100 GPU.

\subsection{Ablation Study}
The main parameters in our method are the ratio for class assignment $\alpha$ and the ratio for pseudo label combination $\beta$. In this subsection, we will investigate the influence of these parameters along with the setting of online zero-shot transfer. The ablation experiments are conducted on ImageNet with ResNet-50~\cite{resnet} as vision encoder in CLIP.

\subsubsection{Effect of $\alpha$} 
$\alpha$ is the ratio to capture the distribution over the entire data set as in Eqn.~\ref{eq:alpha}. While the vanilla zero-shot transfer ignores the side information from the arrived examples, the proposed online label learning can leverage the information of distribution by $\alpha$. Since $\alpha\in[0,1]$, we vary $\alpha$ in $\{0,0.4,0.6,0.8,1\}$ and summarize the results in \cref{ta:alpha}.

\begin{minipage}{\textwidth}
\begin{minipage}[p]{0.4\textwidth}
\centering
\captionof{table}{Effect of $\alpha$ in Eqn.~\ref{eq:alpha} for online label learning on ImageNet with ResNet-50 as vision encoder. ``\#Min class'' denotes the proportion of examples in the smallest class.}\label{ta:alpha}
\begin{tabular}{|l|c|c|}\hline
$\alpha$&Acc(\%)&\#Min class(\%)\\\hline
0&60.32&13.7\\
0.4&60.41&18.5\\
0.6&60.63&23.0\\
0.8&61.20&33.3\\
1&61.96&39.0\\\hline
\end{tabular}
\end{minipage}
\begin{minipage}[p]{0.54\textwidth}
\centering
\includegraphics[width = \textwidth]{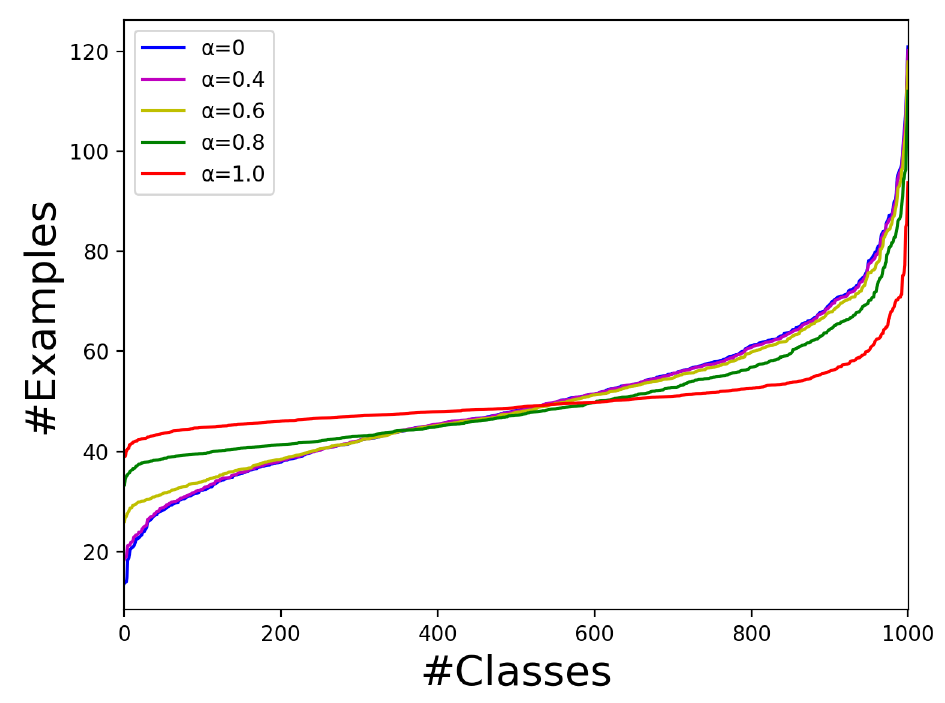}
\captionof{figure}{Illustration of data distribution over $1,000$ classes with different $\alpha$ on ImageNet. Best viewed in color.}\label{fig:alpha}
\end{minipage}
\end{minipage}
\\

First, according to our analysis, $\alpha=0$ will not update dual variables and implies the accuracy of baseline zero-shot transfer as in CLIP. Then, increasing $\alpha$ can consistently improve the accuracy over the baseline, which shows the efficacy of online label learning. Moreover, the size of the smallest class is also increased as $\alpha$ increases and it shows that our proposed method OnLab with only the online label learning can adjust the class assignment effectively. 

To further investigate the data distribution derived by different $\alpha$, we compare the size of all $1,000$ classes in \cref{fig:alpha}. Evidently, the proposed method is capable of balancing the distribution in an online manner. Note that even with $\alpha=1$, the obtained distribution is not perfectly balanced, which is from the approximation error in the online scenario as analyzed in Theorem~\ref{thm:1}. We find that our method is not sensitive to $\alpha$ and fix $\alpha=1$ in the following experiments if not specified. 

\subsubsection{Effect of $\beta$} 

With the predictions from text and vision space in online zero-shot transfer, a combination strategy with $\beta$ in Eqn.~\ref{eq:beta} can trade-off between the bias and variance. \cref{ta:beta} reports the performance with different $\beta$.

\begin{minipage}{\textwidth}
\begin{minipage}{0.57\textwidth}
\centering
\captionof{table}{Effect of $\beta$ in Eqn.~\ref{eq:beta} for online proxy learning on ImageNet with ResNet-50 as vision encoder.}\label{ta:beta}
\begin{tabular}{|l|c|c|c|c|c|c|}\hline
$\beta$&0&0.2&0.4&0.6&0.8&1\\\hline
Acc(\%)&61.96&62.17&62.23&62.38&62.50&61.84\\\hline
\end{tabular}
\end{minipage}
\begin{minipage}{0.37\textwidth}
\centering
\captionof{table}{Comparison of strategies for obtaining the combination ratio $\lambda$ in Eqn.~\ref{eq:lambda}.}\label{ta:lambda}
\begin{tabular}{|l|c|c|}\hline
$\beta$&Fixed&Dynamic\\\hline
Acc(\%)&62.28&62.50\\\hline
\end{tabular}
\end{minipage}
\end{minipage}
\\
\\

When $\beta=0$, only the label from the online label learning (\ie, OnLab) is adopted for prediction. If mixing the label from the vision proxy, the accuracy improves. The mixture prediction outperforms OnLab by $0.54\%$ with $\beta=0.8$ and is $2.18\%$ better than the baseline CLIP. If excluding the label from text space at the end of data streaming as $\beta=1$, the performance degenerates due to the large variance from the vision space. The phenomenon illustrates the challenge of online learning that the vision proxy may not be well learned with a limited number of arrived images and the pseudo label from text proxy can help reduce the variance.

According to our analysis in Section~\ref{sec:onproxy}, a dynamic ratio for the bias-variance trade-off is more effective for the proposed online learning scenario. To verify the claim, we compare the dynamic ratio to the fixed one in \cref{ta:lambda}. The value of fixed $\lambda$ is searched in $[0,1]$ by a step size of $0.1$ and the best result is reported.

While the fixed $\lambda$ still improves the performance from the online label learning, a monotonically increasing $\lambda$ captures the training dynamics better and shows an additional gain of $0.22\%$. It confirms our analysis about the bias-variance trade-off.

\subsubsection{Effect of iterations in optimization}

Our method has to learn the dual variables and vision proxy with streaming data and those parameters may not be learned well at the early stage of online learning. To investigate the influence of iterations in online learning, we report the accumulated accuracy of different iterations in one epoch on ImageNet in \cref{ta:iter}.

\begin{table}[htbp]
\centering
\begin{minipage}[t]{0.35\textwidth}
\centering
\caption{Comparison of accuracy (\%) with different iterations on ImageNet.}\label{ta:iter}
\begin{tabular}{|l|c|c|c|}\hline
\#Iters&$5k$&$10k$&$20k$\\\hline
Acc(\%)&60.96&61.50&61.71\\\hline
&$30k$&$40k$&$50k$\\\hline
&62.02&62.41&62.50\\
\hline
\end{tabular}
\end{minipage}
\begin{minipage}[t]{0.6\textwidth}
\centering
\caption{Effect of epochs for optimization on ImageNet. Prediction from the last epoch is evaluated for OnZeta after running multiple epochs.}\label{ta:epoch}
\resizebox{\linewidth}{!}{
\begin{tabular}{|l|c|c|c|c|c|}\hline
\#Epoch&1(online learning)&2&3&4&5\\\hline
Acc(\%)&62.50&63.24&63.35&63.39&63.46\\\hline
\end{tabular}}
\end{minipage}
\end{table}

First, we can observe that even with $5,000$ iterations for learning, the accuracy over $1,000$ classes is already better than the baseline. With the arrival of more images, the performance increases steadily. It is because that the convergence rate for dual variables and vision proxy is $\OO(1/\sqrt{n})$ as analyzed in Theorems~\ref{thm:1}-\ref{thm:2}. It implies that these variables can reach a reasonable performance with a limited number of iterations, while it is challenging to obtain the optimal solution as discussed in the following experiment.

\subsubsection{Effect of epochs in optimization}
In the proposed online zero-shot transfer problem, all examples can be visited only once when they arrive. The setting is prevalent in real-world applications but very challenging for optimization. This ablation experiment aims to investigate the performance gain when each example can be visited multiple times.

Concretely, we run the proposed method on multiple epochs of data, where images arrive in a random order within each epoch. For multiple epochs, the prediction from the last one is reported for evaluation. \cref{ta:epoch} compares the performance of multiple epochs to that of a single epoch in the online scenario.

First, we can observe that with an additional epoch, the performance of OnZeta can be boosted by $0.74\%$, which demonstrates the challenge of online zero-shot transfer. When increasing the number of epochs, the performance can be further improved and it achieves $63.46\%$ with $5$-epoch learning. The result approaches the performance of learning with the entire data set as in \cite{inmap} that shows a $63.74\%$ accuracy when visiting all data by $2,000$ epochs. The competitive performance of OnZeta confirms the effectiveness of the proposed online learning algorithm. More experiments can be found in the appendix.

\subsection{Comparison on ImageNet}
After the ablation study, we compare our method OnZeta to state-of-the-art methods on ImageNet. Multiple vision encoders from CLIP are adopted for evaluation, which include ResNet-50, ViT-B/32, ViT-B/16, ViT-L/14, and ViT-L/14@336. Considering that images come in a random order in our online setting, the proposed method is repeated by $5$ times and the averaged result is reported for the comparison. We note that the standard deviation is quite small and also report the best result among different runs by $^*$. The same parameters are shared by different vision encoders in our method. \cref{ta:imagenet} summarizes the comparison, where the performance of InMaP~\cite{inmap} is marked by gray due to the application of the whole unlabeled data set. All other methods only access a single image at each iteration. Our OnLab with only the online label learning is denoted as ``OnZeta$_{0.5}$'' and the vanilla zero-shot transfer in Eqn.~\ref{eq:vanilla} is ``Baseline''.

\begin{table}[htbp]
\centering
\caption{Comparison of accuracy (\%) on ImageNet with different vision encoders in CLIP. The overall best performance is in bold. ``-'' denotes that the result is unavailable in their original papers. The averaged performance over $5$ random trials is reported for our method, where the best one is denoted by $^*$. InMaP is marked by gray since it leverages the whole unlabeled image set at each iteration.}\label{ta:imagenet}
\resizebox{\linewidth}{!}{
\begin{tabular}{|l|c|c|c|c|c|c|c|}\hline
Vision encoder&Baseline~\cite{clip}&TPT~\cite{tpt}&OnZeta$_{0.5}$&OnZeta$_{0.5}^*$&OnZeta&OnZeta$^*$&\color{da}InMaP~\cite{inmap}\\\hline
ResNet-50&60.32&60.74&61.89&61.96&62.42&\textbf{62.50}&\color{da}63.74\\
ViT-B/32&63.77&-&65.06&65.18&65.71&\textbf{65.80}&\color{da}67.29\\
ViT-B/16&68.75&68.98&70.18&70.23&70.87&\textbf{70.99}&\color{da}72.55\\
ViT-L/14&75.96&-&76.98&77.07&77.84&\textbf{77.94}&\color{da}79.29\\
ViT-L/14$_{336}$&77.02&-&77.93&78.05&78.75&\textbf{78.94}&\color{da}80.21\\\hline
\end{tabular}}
\end{table}

First, by capturing the overall distribution using only the proposed online label learning component, OnZeta$_{0.5}$ already gains about $1\%$ for all vision encoders over the baseline. It demonstrates that the side information from the arrived images is useful for improving the zero-shot transfer. By learning the vision proxy and combining the predictions from both the text and vision spaces, OnZeta shows about $1\%$ additional improvement over OnZeta$_{0.5}$, which illustrates that online proxy learning is complementary to online label learning for zero-shot classification. Moreover, OnZeta obtains more gains with a large backbone, showing its potential for large models.

Then, we can observe that the prompt learning method TPT outperforms the baseline with an ensemble of artificial text prompts, which indicates the importance of the appropriate prompt for zero-shot transfer. However, OnZeta surpasses TPT by $1.76\%$ on ResNet-50 and $2.01\%$ on ViT-B/16. It is because that we learn a better class proxy in the vision space for visual classification, which helps reduce the modality gap and is more effective than optimizing the text prompt in the text space. The observation is consistent with the analysis in \cite{inmap}.  Moreover, compared with InMaP that can access the whole unlabeled set, OnZeta is only about $1\%$ worse using different vision encoders. As analyzed in the ablation study, online learning that visits each example once without storing is more challenging than the offline method with multiple iterations over the entire set. Nevertheless, online learning is practical for real-world applications. Finally, all vision encoders share the same parameters in experiments. It implies that our method OnZeta is not sensitive to the selection of vision encoder and is applicable in different configurations.

\subsection{Comparison on Other 13 Downstream Tasks}
Besides the comparison on ImageNet, we conduct experiments on other 13 downstream tasks to evaluate the performance of our method. The benchmark data sets for zero-shot transfer are adopted in the comparison, which includes Aircraft~\cite{maji2013fine}, Caltech101~\cite{fei2004learning}, Stanford Cars~\cite{krause20133d}, CIFAR-10~\cite{krizhevsky2009learning}, CIFAR-100~\cite{krizhevsky2009learning}, CUB200-2011~\cite{wah2011caltech}, Describable Textures Dataset (DTD)~\cite{cimpoi2014describing}, EuroSAT~\cite{HelberBDB19}, Flowers~\cite{NilsbackZ08}, Food101~\cite{BossardGG14}, Oxford-IIIT Pet (Pets)~\cite{parkhi2012cats}, Sun397~\cite{XiaoHEOT10}, and UCF101~\cite{abs-1212-0402}. A broad range of downstream tasks are involved by these diverse data sets, \eg, low-resolution image classification, fine-grained visual categorization (FGVC), land cover classification with satellite images, scene categorization, texture recognition, \etc. These tasks contain fewer images than ImageNet, which increases the challenge of our online transfer with the convergence rate of $\OO(1/\sqrt{n})$. Therefore, we reduce $\beta$ from $0.8$ to $0.4$ for all tasks to exploit the label from the text space. For data sets that the baseline already achieves satisfied performance, \eg, Caltech101 and CIFAR-10 with ViT, online label learning can be skipped by letting $\alpha=0$, while $\alpha=0.4$ can provide a slight improvement. Other parameters on the rest data sets remain the same as those for ImageNet. Two distinct vision encoders, \ie, ResNet-50 and ViT-B/16, are applied for evaluation. The performance averaged over $5$ random trials is reported for our method.

\begin{table}[htbp]
\centering
\caption{Comparison of accuracy (\%) on 13 downstream tasks with ResNet-50 and ViT-B/16 as vision encoders in CLIP. The overall best performance is in bold. ``-'' denotes that the result is unavailable in their original papers. The averaged performance over $5$ random trials is reported for our method, where the best one is denoted by $^*$. InMaP is marked by gray since it leverages the whole unlabeled image set at each iteration.}\label{ta:other}
\resizebox{\linewidth}{!}{
\begin{tabular}{|l|c|c|c|c|c|c|c|c|c|c|c|c|c|c|}\hline
&Aircraft&Caltech&Cars&Cifar10&Cifar100&CUB&DTD&EuroSAT&Flowers&Food&Pets&SUN&UCF101&Avg.\\\hline
\multicolumn{13}{|l|}{\textit{ResNet-50:}}\\\hline
Baseline&16.62&86.00&56.31&73.15&40.60&41.37&41.13&26.90&62.97&74.10&81.85&59.25&55.56&55.06 \\
TPT&17.58&\textbf{87.02}&58.46&-&-&-&40.84&28.33&62.69&74.88&84.49&61.46&60.82&-\\
OnZeta$_{0.5}$&18.29&86.17&58.57&75.62&46.96&45.68&43.24&33.79&62.42&76.41&83.05&61.64& 60.12&57.84\\
OnZeta$_{0.5}^*$&18.51&86.45&58.91&75.75&47.22&45.93&43.74&34.32&63.17&76.53&83.59&61.81&60.38&58.18\\
OnZeta&18.31&86.26&59.40&75.84&47.55&46.06&43.38&33.99&63.16&76.83&83.77&62.32& 61.00&58.30\\
OnZeta$^*$&\textbf{18.54}&86.41&\textbf{59.71}&\textbf{75.95}&\textbf{47.71}&\textbf{46.34}&\textbf{43.91}&\textbf{34.35}&\textbf{63.95}&\textbf{76.97}&\textbf{84.66}&\textbf{62.40}&\textbf{61.38}&\textbf{58.64}\\
\color{da}InMaP&\color{da}18.96&\color{da}86.73&\color{da}63.30&\color{da}78.84&\color{da}49.26&\color{da}49.17&\color{da}44.86&\color{da}35.88&\color{da}66.68&\color{da}78.36&\color{da}87.93&\color{da}63.82&\color{da}63.36&\color{da}60.55\\
\hline
\multicolumn{13}{|l|}{\textit{ViT-B/16:}}\\\hline
Baseline&23.13&93.51&66.29&91.09&67.29&49.36&45.09&50.17&66.99&84.43&87.00&65.68&65.21&65.79 \\
TPT&24.78&\textbf{94.16}&66.87&-&-&-&47.75&42.44&68.98&84.67&87.79&65.50&68.04&-\\
OnZeta$_{0.5}$&27.19&93.65&67.99&91.13&69.85&52.81&47.51&56.20&68.49&85.83&88.44&68.06&68.50&68.13\\
OnZeta$_{0.5}^*$&27.48&93.79&68.11&91.18&70.02&53.09&47.87&56.42&69.06&85.97&88.72&68.27&68.86&68.37\\
OnZeta&27.57&93.69&68.74&91.31&70.39&53.66&48.04&56.52&69.37&86.21&89.01&68.87&69.58&68.69\\
OnZeta$^*$&\textbf{28.29}&93.83&\textbf{69.03}&\textbf{91.35}&\textbf{70.52}&\textbf{53.80}&\textbf{48.58}&\textbf{56.74}&\textbf{69.63}&\textbf{86.35}&\textbf{89.32}&\textbf{69.01}&\textbf{69.94}&\textbf{68.95}\\
\color{da}InMaP&\color{da}28.35&\color{da}93.59&\color{da}73.00&\color{da}93.27&\color{da}72.51&\color{da}57.87&\color{da}50.77&\color{da}63.98&\color{da}71.28&\color{da}87.76&\color{da}92.94&\color{da}70.85&\color{da}72.06&\color{da}71.40\\
\hline
\end{tabular}}
\end{table}

The comparison is summarized in \cref{ta:other}. A similar phenomenon as ImageNet can be observed on these downstream tasks. Concretely, online label learning, \ie, OnZeta$_{0.5}$, already improves the performance over baselines, where the gain is $3.12\%$ on ResNet and $2.58\%$ on ViT, respectively. Compared with the baseline, only streaming images are leveraged by our method, and the representation of each arrived image is not stored, which preserves the flexibility of zero-shot transfer but captures the distribution over the entire data set in an online manner. In addition, learning the vision proxy with the pseudo label from the text space further improves the performance with different vision encoders, which demonstrates the effectiveness of learning the class proxy directly in the target vision space. Moreover, OnZeta outperforms TPT on $9$ out of $10$ (only 10 tasks are available in TPT's original paper) data sets with different vision encoders. Compared with TPT that has to optimize the text prompt with multiple augmentations of the image by back-propagation through the text encoder, OnZeta only contains the forward pass to extract the representation of the image, which is more efficient and is capable of online deployment. By comparing the best performance of the proposed method (\ie, indicated by * in the table) to the averaged one, we can find that the gap averaged over 13 data sets is only about $0.3\%$. While the images arrive in a random order, the performance of the proposed method can be theoretically guaranteed as indicated in Theorems~\ref{thm:1}-\ref{thm:2}, which is robust to the order of input images. Finally, most data sets share the same parameters and it implies that OnZeta is not sensitive to hyper-parameters and is applicable for different tasks.

\section{Conclusion}
\label{sec:conclusion}
While CLIP demonstrates impressive zero-shot transfer performance, the information from the target data has not been explored sufficiently. In this work, we investigate an online zero-shot scenario, where a random image arrives at each iteration and the model has to predict its label without storing its representation. By capturing the overall label distribution and refining the class proxy in an online manner, our method OnZeta demonstrates a substantial improvement over the vanilla zero-shot transfer on diverse downstream tasks with different vision encoders. In addition, the operations in the proposed method can be implemented efficiently for real-time applications. Exploring more effective algorithms for challenging online zero-shot transfer can be our future work. 

\section*{Limitations}
In this work, we only consider a single vison-language pre-trained model for online zero-shot transfer, while other pre-trained models, \eg, large language models and diffusion models, can be leveraged to mitigate the challenges from the vision-language models.

%
%
\bibliographystyle{splncs04}
\bibliography{onlineclip}
\appendix
\section{Theoretical Analysis}
\subsection{Proof of Theorem~1}
\begin{proof}
According to the optimality of the closed-form solution, we have
\begin{align}\label{eq:p}
\forall p_i, \quad \LL(\rho^{i-1}, p_i^*)\leq \LL(\rho^{i-1}, p_i)
\end{align}
With the gradient ascent of dual variables, we have
\begin{align}\label{eq:rho}
&\LL(\rho, p_i^*) - \LL(\rho^{i-1}, p_i^*)= \sum_j (\rho_j^{i-1} - \rho_j)(p_{i,j}^* - \frac{\alpha}{C})\nonumber\\
&\leq \frac{\|\rho^{i-1} - \rho\|_2^2 - \|\rho^i-\rho\|_2^2}{2\eta_\rho^i} +\frac{\eta_\rho^i}{2}
\end{align}
Combining Eqns.~\ref{eq:p} and \ref{eq:rho}, we have
\[\LL(\rho, p_i^*) - \LL(\rho^{i-1}, p_i)\leq \frac{\|\rho^{i-1} - \rho\|_2^2 - \|\rho^i-\rho\|_2^2}{2\eta_\rho^i} +\frac{\eta_\rho^i}{2}\]
Summing $i$ from $1$ to $n$, we have
\begin{align*}
&\sum_i \LL(\rho, p_i^*) - \LL(\rho^{i-1}, p_i)\leq \frac{1}{2}(\sum_i \|\rho^{i-1}-\rho\|_2^2(\frac{1}{\eta_\rho^i} - \frac{1}{\eta_\rho^{i-1}})+\sum_i\eta_\rho^i)
\end{align*}
By assuming $\|\rho\|_2\leq \gamma$ and $\eta_\rho^i = c_\rho/\sqrt{i}$, we have
\[\sum_i \LL(\rho, p_i^*) - \LL(\rho^{i-1}, p_i)\leq \gamma^2\sqrt{n}/c_\rho+c_\rho\sqrt{n}\]
\end{proof}
\subsection{Proof of Theorem~2}
\begin{proof}
Due to the convexity of the loss function, we have
\begin{align*}
&\LL(\w^{i-1}, \x_i) - \LL(\w, \x_i) \leq \langle \nabla \LL(\w^{i-1}, \x_i)), \w^{i-1} - \w\rangle\\
&\leq \frac{\|\w^{i-1}-\w\|_F^2 - \|\w^{i}-\w\|_F^2}{2\eta_w^i} +\frac{\eta_w^i}{2} \|\nabla \LL(\w^{i-1}, \x_i))\|_F^2
\end{align*}
Assuming the norm of the gradient is bounded by $\sigma$ as $\forall i,\|\nabla \LL(\w^{i-1}, \x_i))\|_F\leq \sigma$ and summing $i$ from $1$ to $n$, we have
\begin{align*}
&\LL(\w^{i-1}, \x_i) - \LL(\w, \x_i) \leq \sum_i2C(\frac{1}{\eta_w^i} - \frac{1}{\eta_w^{i-1}})+c_w\sigma^2\sqrt{n}\\
&\leq 2C\sqrt{n}/c_w+ c_w\sigma^2\sqrt{n}
\end{align*}
\end{proof}

\section{Experiments}
\subsection{Ablation Study}
\subsubsection{Effect of text prompts}
Besides the $7$ text prompts for the ensemble, we also benchmark the performance with a single prompt as ``a photo of a \{class name\}'' for the comparison in \cref{ta:prompt}. The performance of ViT-B as vision encoder is also evaluated. We denote the vanilla zero-shot transfer in CLIP as ``Baseline''.

\begin{table}[htbp]
\centering
\caption{Comparison of accuracy (\%) for different text prompt strategies on ImageNet. A single prompt is applied as ``a photo of a \{class name\}'' while $7$ prompts are suggested in \cite{clip,inmap}.}\label{ta:prompt}
\begin{tabular}{|l|c|c|c|c|}\hline
Method&Single prompt & 7 prompts&Single prompt & 7 prompts\\\hline
\multicolumn{3}{|l|}{\textit{ResNet-50:}}&\multicolumn{2}{|l|}{\textit{ViT-B/16:}}\\\hline
Baseline&58.15&60.32&66.72&68.75\\\hline
OnZeta&61.49&62.50&70.01&70.99\\\hline
\end{tabular}
\end{table}

First, we can observe that the ensemble of prompts outperforms the single prompt with a clear margin of about $2\%$ on the baseline. It illustrates that the vanilla zero-shot transfer is sensitive to the selection of text prompts. With the proposed online learning strategy, OnZeta outperforms the baseline with both the single prompt and 7 prompts. Concretely, it surpasses the baseline by more than $3\%$ on the single prompt with different vision encoders and is more than $2\%$ better than the baseline with 7 prompts. The superior performance of OnZeta demonstrates the effectiveness of side information from the arrived examples in the proposed online scenario. Finally, the gap between different prompts shrinks by $1\%$ when applying OnZeta and it shows that our method is less sensitive to the text prompts for obtaining the text proxy.

\subsubsection{Comparision of similarity with different proxies}

Following \cite{inmap}, we have the averaged cosine similarity between each image and its nearest class proxy to depict the modality gap. First, given text proxy from class names, the similarity between images and the text proxy is only $0.26$. After online learning by OnZeta, that between examples and the vision proxy increases to $0.39$, which shows the efficacy of mitigating the modality gap.

\subsubsection{Comparison on different pre-trained models}

Besides CLIP~\cite{clip}, there are many other pre-trained vision-language models. \cref{ta:model} summarizes the performance of our method with backbones from ALBEF~\cite{albef} and BLIP~\cite{blip}, respectively. Compared with CLIP, OnZeta achieves even larger improvement than the standard zero-shot classification baseline. It shows that our method is applicable to different pre-trained models.

\begin{table}[htbp]
\centering
\caption{Comparison of accuracy (\%) on ImageNet with different pre-trained models.}\label{ta:model}
\begin{tabular}{|l|c|c|}\hline
Models&ALBEF~\cite{albef}&BLIP~\cite{blip}\\\hline
Baseline&33.82&47.17\\\hline
OnZeta&41.86&54.85\\\hline
\end{tabular}
\end{table}

\subsubsection{Comparison on variants of ImageNet}

We evaluate the performance of OnZeta on variants of ImageNet with shifted distributions, which include ImageNet-V2~\cite{imagev2}, ImageNet-A~\cite{imagea} and ImageNet-Sketch~\cite{images}, as shown in ~\cref{ta:robust}. With appropriate parameters, our method consistently improves the baseline and it demonstrates the robustness of OnZeta.

\begin{table}[htbp]
\centering
\caption{Comparison of accuracy (\%) for different variants of ImageNet with ViT-L as vision encoder.}\label{ta:robust}
\begin{tabular}{|l|c|c|c|}\hline
Method&ImageNet-V2 & ImageNet-A&ImageNet-Sketch\\\hline
Baseline&70.21&70.85&59.70\\\hline
OnZeta&71.23&72.61&62.71\\\hline
\end{tabular}
\end{table}

\subsubsection{Running time} Finally, we demonstrate the efficiency of online learning. For ImageNet with ResNet-50 as the vision encoder, the total cost including feature extraction for zero-shot CLIP is 1.74 (milliseconds) per image and that of our method is 2.12 (ms). Both can work for real-time applications and the overhead introduced by OnZeta is mild.

\end{document}